\documentclass{article}

\usepackage[utf8]{inputenc} 
\usepackage[T1]{fontenc}    
\usepackage{hyperref}       
\usepackage{url}            
\usepackage{booktabs}       
\usepackage{amsfonts}       
\usepackage{nicefrac}       
\usepackage{microtype}      
\usepackage{xcolor}         

\usepackage{cite}

\usepackage{amsmath,amssymb}
\usepackage{algorithmic}
\usepackage{algorithm}
\usepackage{graphicx}
\usepackage{textcomp}

\usepackage{wrapfig,lipsum}

\usepackage{dsfont}

\usepackage{algorithm}

\usepackage{caption}
\usepackage{subcaption}
\usepackage{amsthm}
\usepackage{subfloat}

\newcommand{\gp}{g_\pi}    
\newcommand{\gr}{g_r}

\newtheorem{theorem}{Proposition}
\newtheorem*{cor}{Corollary}
\graphicspath{ {./graphics/} }
\usepackage[preprint]{corl_2022} 

\title{USHER: Unbiased Sampling for\\ Hindsight Experience Replay}

\author{%
  Liam Schramm \\
  Department of Computer Science\\
  Rutgers University\\
  New Brunswick, NJ 08904 \\
  \texttt{lbs105@rutgers.edu} \\
   \And
  Yunfu Deng \\
  Department of Computer Science\\
  Rutgers University\\
  New Brunswick, NJ 08904 \\
  \texttt{yd275@scarletmail.rutgers.edu} \\
   \And
  Edgar Granados \\
  Department of Computer Science\\
  Rutgers University\\
  New Brunswick, NJ 08904 \\
  \texttt{gary.granados@gmail.com} \\
  \And
  Abdeslam Boularias \\
  Department of Computer Science\\
  Rutgers University\\
  New Brunswick, NJ 08904 \\
  \texttt{ab1544@rutgers.edu} \\
}

\begin{document}
\maketitle


\begin{abstract}
  Dealing with sparse rewards is a long-standing challenge in reinforcement learning (RL). Hindsight Experience Replay (HER) addresses this problem by reusing failed trajectories for one goal as successful trajectories for another. This allows for both a minimum density of reward and for generalization across multiple goals. However, this strategy is known to result in a biased value function, as the update rule underestimates the likelihood of bad outcomes in a stochastic environment. We propose an asymptotically unbiased importance-sampling-based algorithm to address this problem without sacrificing performance on deterministic environments. We show its effectiveness on a range of robotic systems, including challenging high dimensional stochastic environments.
\end{abstract}

\vspace{-0.2cm}
\section{Introduction}
\vspace{-0.1cm}
In recent years, model-free reinforcement learning (RL) has become a popular approach in robotics. In particular, these methods stand out in their ability to learn near-optimal policies in high-dimensional spaces~\citep{DQN,AlphaGoZero,CURL}. 
One popular extension of RL, {\it multi-goal RL}, allows trained robots to generalize to new tasks by conditioning on a goal parameter that determines the reward function. 
However, RL algorithms often struggle with tasks that involve sparse rewards, as these environments can require a very large amount of exploration to discover good solutions. 
Hindsight Experience Replay (HER) offers a solution to the sparse reward problem for multi-goal reinforcement learning~\citep{HER}. 
\begin{wrapfigure}[14]{r}{.35\linewidth}
\vspace{-0.4cm}
		\centering
\begin{minipage}{.5\linewidth}
  \centering
  \includegraphics[width=.9\linewidth]{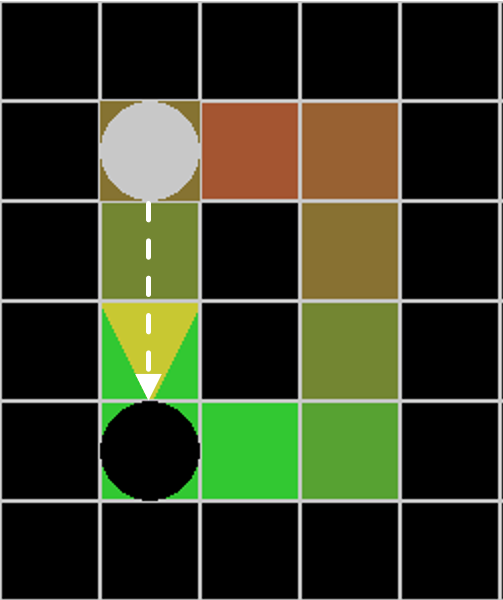}
\end{minipage}%
\begin{minipage}{.5\linewidth}
  \centering
  \includegraphics[width=.9\linewidth]{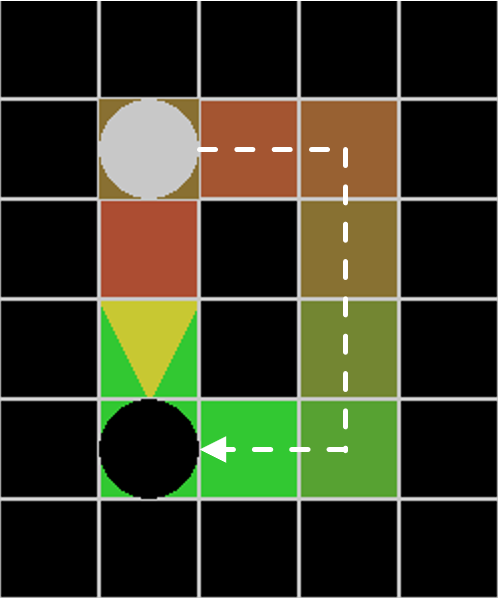}
\end{minipage}
\vspace{-0.2cm}
\caption{\small Q-values learned with HER (left), and Q-learning (right). A robot must navigate from the white circle to the black circle while avoiding obstacles (black squares) and risky areas (yellow triangle, $75\%$ chance of stopping the robot). The value function ranges from 1 (bright green) to 0 (bright red).}
\label{fig:test}
\end{wrapfigure} 
HER treats failed attempts to reach one goal as successful attempts to reach another goal. This significantly reduces the difficulty of the exploration problem, because it guarantees a minimum density of reward and ensures that every trajectory receives useful feedback on how to reach some goal, even when the reward signal is sparse. However, these benefits come with a trade-off. While HER is unbiased in deterministic environments, it is known to be asymptotically biased in stochastic environments~\citep{UnbiasedMethods,CallForResearch}. This is because HER suffers from a {\it survivorship bias}. Since failed trajectories to one goal are treated as successful trajectories to another, it follows that HER only ever sees successful trajectories. If a random event can prevent the robot from reaching a desired goal $g$, then HER will only sample $g$ as a goal when the event did not occur, leading it to significantly overestimate the likelihood of success and underestimate the likelihood of dangerous events.  Practically, this manifests as a tendency for HER to want to ``run red lights'' and take risks. 

We present a concrete toy example of this problem in Figure~\ref{fig:test}, using tabular Q-learning. As we can see, HER values the direct path to the goal and the square en route to the dangerous square much higher than that path's correct Q-value because it undersestimates the risk. HER learns to take the shorter, more dangerous path and achieves a lower success rate with lower reward than Q-learning.

As suggested in both \citep{UnbiasedMethods} and \citep{CallForResearch}, we derive an approach that allows us to use HER for sampling goals without suffering from these bias problems. We do this by separating the goal used for the reward function ($\gr$) from the goal that is passed to the policy ($\gp$). The value function is conditioned on both goals, but only the reward goal is sampled using HER. This allows us to efficiently learn a successor representation over future achieved goals that we can use for importance sampling. We show that reweighting HER's mean squared Bellman error using this successor representation yields an unbiased estimate of the error.  We call this method Unbiased Sampling for Hindsight Experience Replay (USHER). We demonstrate this approach on an array of stochastic environments, and find that it counteracts the bias shown by HER without compromising HER's sample efficiency or stability.

\section{Definitions}
We define a {\it multi-goal} Markov Decision Process (MDP) as a seven-tuple: state space $S\subseteq \mathbb{R}^n$, action space $A\subseteq\mathbb{R}^m$, discount factor $\gamma\in [0,1]$, transition probability distribution $P(s' \mid s, a)$ (with density function $f(s' \mid s, a)$) for $(s,a,s') \in S\times A\times S$, goal space $G\subseteq\mathbb{R}^l$, goal function $\phi : S \rightarrow G$, and reward function $R: S\times G \rightarrow \mathbb{R}$. A goal $g = \phi(s)\in G$ is a vector of goal-relevant features of state $s\in S$. Goal function $\phi$ is defined a priori, depending on the task. 
A typical example of $\phi(s)$ is a low-dimensional vector that preserves only the entries of state-vector $s$ that are relevant to the goal.
For instance, a mobile robot is tasked with moving to a particular location and arriving there at zero velocity. The state space of the robot would include velocities and orientations of each wheel, along with several other attributes that are needed to control the robot. The goal function would take the full high-dimensional state of the robot and return only its location  and velocity. Therefore, each goal point corresponds to a subspace of the state space in this example. A special case is when $G=S$ and $g=\phi(s), \forall s\in S$.
Note that the immediate reward function $R(s, g)$ depends on a selected goal $g\in G$. Every selection of $g\in G$ produces a valid single-goal MDP. 
We denote by $\pi$ a deterministic goal-conditioned policy, with $\pi(s, g) \in A$ for $s \in S, g \in G$, and define $Q^*(s, a, g)$ to be the unique optimal $Q$-value of action $a\in A$ in state $s \in S$, given selected goal $g \in G$.

In the proposed algorithm and analysis, a policy $\pi$ can be evaluated according to a goal that is not necessarily the same goal used by the policy for selecting actions. Therefore, we use $\gp$ to refer to goals that are passed to policies, and $\gr$ to denote goals that are used to evaluate policies. Using these notations, the Bellman equation is re-written as 
\begin{align*}
    Q^{\pi}(s, a, \gr, \gp) = \mathbb{E}_{s'}[R(s', \gr) + \gamma Q^{\pi}(s', \pi(s', \gp), \gr, \gp) \mid s, a].
\end{align*}
Intuitively, this means ``The expected cumulative discounted sum of rewards $R(s' , \gr)$, when using policy $\pi(s', \gp)$''. 
The reason for this separation is that it allows us to more easily separate the problem of predicting future rewards from the problem of directing the policy. This makes it much easier to find an analytic expression for HER's bias. In particular, it lets us learn an expression for future goal occupancy that is conditioned only on $\gp$ and not $\gr$, which will allow us to correct for the bias induced by hindsight sampling. 
Observe that when $\gr=\gp$, this definition reduces to the Bellman equation for standard multi-goal RL. 
For standard $Q$-learning, $\pi(s',\gp)$ would be $ \arg\max_{a'} Q(s', a', \gp, \gp)$, where both the policy and reward goals are set to $\gp$. 

{\bf HER.}
 \label{Sec:HER}
		HER is a modification of the experience replay method employed by many deep RL algorithms~\cite{HER,DQN,DDPG,SAC,TD3}. 
		Policy goal $\gp$ is sampled before each trajectory begins, and is not changed while generating the trajectory.
		After generating a trajectory, HER stores the entire trajectory in the replay buffer. When sampling transitions $(s, \gp, a, s')$ from the buffer, HER retains the original goal $\gp$ used in the policy that generated the trajectory, i.e., $\gr\leftarrow \gp$, with probability $\frac{1}{k+1}$, where $k$ is a natural number (usually 4 or 8). The rest of the time, it replaces the original goal with $\phi(s_t)$, i.e., $\gr\leftarrow \phi(s_t)$, where $s_t$ is a randomly sampled state from the future trajectory that starts at $s$. Goals that are selected from the future trajectory are referred to as ``hindsight goals''. 
HER then updates the Q-value and policy networks with $(s, \gp, a, s', R(s',\gr))$.
%

\section{Related Work}
\label{related}
		Over the last few years, several methods have attempted to address the hindsight bias induced by HER. ARCHER attempts to decrease HER's hindsight bias by multiplying the loss on hindsight goals and non-hindsight goals by different weights, effectively upweighting the importance of hindsight goals 
 		\cite{ARCHER}. MHER combines a multi-step Bellman equation with a bias/variance tradeoff equation to address the bias induced by the multi-step algorithm  \cite{MHER}. It is worth noting that MHER only attempts to address HER's off-policy bias, not its hindsight bias.  A rigorous mathematical approach to HER's hindsight bias is taken in~\cite{UnbiasedMethods}, by showing that HER is unbiased in deterministic environments, and that one of HER's key benefits is ensuring a minimum density of feedback from the reward function, even in high-dimensional spaces where the reward density would normally be extremely low. 
        This reward-density problem is addressed by deriving a family of algorithms (called the $\delta$-{\it family}, e.g. $\delta$-DQN, $\delta$-PPO), which guarantees a minimum reward density while still being unbiased. 
		These methods do not use HER and have higher variance. 
		The authors of~\cite{UnbiasedMethods} also state that the problem of formulating an unbiased form of HER is still open, and call for additional research into the problem.

		Bias-Corrected HER (BHER) attempts to account for hindsight bias by analytically calculating importance-sampling hindsight goals 
		\cite{BHER}. Unfortunately, we believe that this derivation is incorrect. 
		The proof in BHER relies on the assumption that the probability of a transition is independent of the goal ($f(s' \mid s, a, g) = f(s' \mid s, a)$). This assumption does not hold for HER, because it samples the goal from the future trajectory of $s$, which depends on $s'$. Both our work and~\cite{UnbiasedMethods} give concrete counterexamples to this assumption. The following derivation provides an unbiased solution that does not rely on this flawed assumption. 

\section{Derivation}
\label{derivation}
{\bf Bias in HER.}
We derive the formula of the bias introduced by HER in estimating the Q-value function in the following. 
Let $s, a$, and $s'$ be random variables representing a state, action, and subsequent state in a given trajectory generated by policy $\pi$ with goal $\gp$. Let $T$ be the number of time-steps remaining in the sub-trajectory that starts at $s$. 
Let $Q^{\pi}_{HER}(s, a, \gr, \gp)$ be the solution to the Bellman equation obtained using HER's sampling process of reward goal  $\gr$ (Sec.~\ref{Sec:HER}). This sampling process takes into account both $\gp$ and $T$. Furthermore, $\gr$ is selected from the sub-trajectory that starts at $s$ with probability $\frac{k}{k+1}$. Therefore, the probability $f(s' \mid s, a, \gr, \gp, T)$  of the next state $s'$ after knowing $\gp, \gr$ and $T$ is generally not the same as $f(s' \mid s, a)$, which is what HER uses empirically to estimate $Q^{\pi}_{HER}(s, a, \gr, \gp)$. The following proposition quantifies this bias ratio.

\begin{theorem}
\label{bias_ratio}
Suppose $\gp$ is fixed at the start of the trajectory, and $\gr$ is sampled using HER. Then for any $s', s, a, \gr, \gp, T$,
 $   f(s' \mid s, a, \gr, \gp, T) = \frac{f(\gr \mid s', \pi(s', \gp), \gp, T-1) }{f(\gr \mid s, a, \gp, T)} f(s' \mid s, a).$
\end{theorem}

\textbf{Proof:} Appendix (A.3). This identity presents an interesting corollary.
\begin{cor}
Suppose $Q^{\pi}_{HER}(s, a, \gp, \gp)$ satisfies the Bellman equation and the distribution of future achieved goals is absolutely continuous with respect to the goal space for all $s, a, \gp$, and $\pi (s,\gp)=\arg\max_{a'} Q^{\pi}_{HER}(s, a', \gp, \gp)$. Then $Q^{\pi}_{HER}(s, a, \gp, \gp) = Q^*(s, a, \gp)$, where $Q^*$ is the optimal goal-conditioned $Q$-function. 
\end{cor}
\textbf{Proof:} Appendix (A.4). While this establishes that the target value for $Q^{\pi}_{HER}$ is unbiased when $\gr=\gp$, the function approximator for $Q^{\pi}_{HER}$ may still be biased, because values other than $\gr=\gp$ may influence it through the training of the network. Thus, it is possible that the learned $Q^{\pi}_{HER}$ value may remain biased until unacceptably large amounts of data are gathered. Additionally, 
since the density of data is discontinuous, $Q^{\pi}_{HER}$ may be discontinuous and difficult to approximate with a neural network. The rest of this section is devoted to developing an importance sampling method that is guaranteed to be asymptotically unbiased over the entire domain of  $Q$. 

{\bf Unbiased HER.}
To estimate $Q^{\pi}(s, a, \gr, \gp)$, the solution to the unbiased Bellman equation, we use in this work the following expression,
\begin{align*}
    Q^{\pi}(s, a, \gr, \gp) = \mathbb{E}_{s'}[M(s', s, a, \gr, \gp, T)\big(R(s', \gr) + \gamma Q^{\pi}(s', \pi(s', \gp), \gr, \gp)\big) \mid s, a, \gr, \gp, T],
\end{align*}
where $M(s', s, a, \gr, \gp, T)$ is a weight that cancels the bias ratio given in Proposition~\ref{bias_ratio}. Conditioning the expected value over $s'$ on $\gr, \gp$, and $T$ frees us from the constraint that $s'$ needs to be independent of $\gr, \gp$, and $T$. This would allow us to select $\gr$ from the future trajectory of $s$, as HER does. Note that conditioning on $T$, the number of steps left in the trajectory, is necessary because the distribution of goals selected by HER is not time-independent. 

Proposition~\ref{bias_ratio} is useful for understanding what situations may cause HER to be biased, but unfortunately we cannot directly use it for importance sampling. Weighting samples by setting $M(s', s, a, \gr, \gp, T)$ as $\frac{f(\gr \mid s, a, \gp, T)}{f(\gr \mid s', \pi(s', \gp), \gp, T-1) }$ would require $f(\gr \mid s', \pi(s', \gp), \gp, T-1) $ to always be greater than 0, which is not necessarily true. To solve this, we sample a mixture of hindsight goals and goals drawn uniformly from the goal space $G$. Of the goals where $\gr \neq \gp$, a fraction $\alpha$ of our goals will be drawn uniformly from the goal space, and the remaining $1-\alpha$ will be drawn from the trajectory that follows $s$. This results in the following identity, 

\begin{theorem}
\label{w_ratio}
Let $W(s', s, a, \gr, \gp, T) = \frac{f(\gr \mid s, a, \gp, T)}{\alpha f(\gr \mid s, a, \gp, T) + (1-\alpha) f(\gr \mid s', \pi(s', \gp), \gp, T-1)}$. 
Let $\alpha$ be a real value in the range $(0,1]$. Then for any $s', s, a, \gr, \gp$, 
\begin{align*}
    f(s' \mid s, a) = W(s', s, a, \gr, \gp, T)\big( \alpha f(s' \mid s, a)+ (1-\alpha) f(s' \mid s, a, \gp, \gr, T)\big) 
\end{align*}
Furthermore, for any function $F$ of state $s'$,
\begin{align}
\label{unbiased_general_bellman}
       \mathbb{E}_{s'}[F(s') \mid s, a]  = \textrm{ } & \alpha \mathbb{E}_{s'}[W(s', s, a, \gr, \gp, T) F(s') \mid s, a] \nonumber \\  + & (1-\alpha)\mathbb{E}_{s'}[W(s', s, a, \gr, \gp, T) F(s') \mid s, a, \gp, \gr, T].
\end{align}
\end{theorem}
\textbf{Proof:} Appendix (A.5). We can now derive an unbiased variant of HER by applying Proposition~\ref{w_ratio} to Bellman equation. 

\begin{cor}
Suppose $\pi$ is a deterministic policy, $\gr$ is sampled from the previously mentioned mix of hindsight and uniform random goals, and $\gp$. Then for any $s', s, a, \gr, \gp, T$,  
\begin{align}
\label{unbiased_bellman}
    Q(s, a, \gr&, \gp)  = \alpha \mathbb{E}_{s'}[W(s', s, a, \gr, \gp, T) \big(R(s', \gr) + \gamma Q(s', \pi(s', \gp), \gr, \gp)\big) \mid s, a] \nonumber \\ 
    + (1-\alpha)&\mathbb{E}_{s'}[W(s', s, a, \gr, \gp, T) \big(R(s', \gr) + \gamma Q(s', \pi(s', \gp), \gr, \gp)\big) \mid s, a, \gp, \gr, T]
\end{align}
\end{cor}
This corollary provides us with a simple method of estimating $Q(s, a, \gr, \gp)$ using HER. A similar unbiased expression can be derived for estimating the gradient of the Bellman error with respect to the weights of a $Q$-function network, instead of estimating $Q(s, a, \gr, \gp)$ directly from samples.

{\bf Learning the future goal distribution.}
In order to use the proposed unbiased estimator of the Q-function  with policy and reward goals, we need to compute weight $W$ defined in Proposition~\ref{w_ratio}. This can be achieved by learning future goal distributions 
$f(\gr \mid s, a, \gp, T)$ and $f(\gr \mid s', \pi(s', \gp), \gp, T-1)$, which both correspond to the conditional probability that a given goal $\gr$ will be selected as a hindsight goal by HER. A technique for learning such long-term distributions, introduced in~\cite{UnbiasedMethods}, consists in training a network $f_{\theta}$, with parameters $\theta$, to approximate the density of future goals $f(\gr \mid s, a)$. The following estimator for the gradient is used in~\cite{UnbiasedMethods}, sampling $(s, a, s')$ from transitions, and fixing $g_r$ at the start of each trajectory,
\begin{align*}
     \nabla_\theta \big(\mathbb{E}_{s, a}[-f_\theta(s, a, \phi(s))] + \mathbb{E}_{s, a, s', \gr}[f_\theta(s, a, \gr) (f_\theta(s, a, \gr) - \gamma \text{max}_{a'} f_{\textrm{target}}(s', a', \gr)]\big),
\end{align*}
wherein $f_{\textrm{target}}$ is a copy of $f_\theta$ that is updated separately. This method has however a significantly higher variance than HER~\cite{UnbiasedMethods}. We examine here the source of this variance, and explain how separating the policy and reward goals allows us to avoid this variance problem. One issue with this method that can contribute to variance is that the gradient is separated into two parts: one in which the goal comes from the state ($\phi(s)$), and one in which the goal is sampled at the start of the trajectory ($\gr$). This is a problem, because the gradient at the state-derived goals is strictly negative, while the gradient at the sampled goals is usually positive. In our experiments, this led to a pattern where the value at the state-derived goals would diverge unboundedly, until a goal was sampled sufficiently close to make the value function crash back down to zero, and then the process would repeat again. In other words, it is not guaranteed that $f_\theta$ converges a fixed point for every finite set of trajectories. 

One way to avoid this problem would be to have a fixed, non-zero chance that $\phi(s) = \gr$, so that $f_\theta$ always converges to a fixed point given any set of training trajectories. We use HER to achieve this outcome. This is possible, unlike in~\citep{UnbiasedMethods}, because we can use the importance sampling method derived above to sample a mixture of HER goals and goals independent of the state. Since HER draws from the future states of $s$, observe that $f(\gr \mid s, a, \gp, T)$ is in fact a {\it successor representation}~\citep{Dayan93}, using an average-reward formulation (because the probability of selecting any of the next T states is uniform). Observe that we can define this probability as 
$$    f(\gr \mid s, a, \gp, T) = \mathbb{E}_{s'} [\frac{1}{T} \delta(\gr - \phi(s')) + (1-\frac{1}{T}) f(\gr \mid s', \pi(s', \gp), \gp, T-1) \mid s, a],$$
wherein $\delta$ is Dirac delta function. This results in the loss gradient:
\begin{align}
\label{loss_gradient}
     \nabla_\theta \Big( \mathbb{E}_{s, a, \gr, \gp, T}[-\frac{2}{T}f_\theta(s, a, \phi(s), \gp, T)  +  \mathbb{E}_{s'} [ L(s, a, s', \gr, \gp, T) \mid s, a]]\Big),
\end{align}
$L(s, a, s', \gr, \gp, T) \triangleq f_\theta(s, a, \gr, \gp, T) \big(f_\theta(s, a, \gr, \gp, T) - \gamma f_{\textrm{target}}(s', \pi(s', \gp), \gr, \gp, T-1)\big)$. While $f_\theta$ may not be a true probability density (because it may not integrate to $1$), this does not matter for our purposes, as this factor will divide out when we calculate $W$.
Finally, we inject the formula in Equation~\ref{loss_gradient} into Equation~\ref{unbiased_general_bellman}, while replacing $F$ with $f_{\theta}$, to derive the following unbiased loss gradient,
\begin{align}
\label{f_loss_gradient}
    \nabla_\theta \mathcal{L} &= \nabla_\theta \mathbb{E}[  -\frac{2}{T}f_\theta(s, a, \phi(s), \gp, T) \nonumber+
    \alpha \mathbb{E}_{s'}[W(s, a, s', \gr, \gp, T) L(s, a, s', \gr, \gp, T) \mid s, a] \nonumber \\
    &+(1-\alpha) \mathbb{E}[W(s, a, s', \gr, \gp, T) L(s, a, s', \gr, \gp, T) \mid s, a, \gr, \gp, T]].
\end{align}

Note that the values of $\alpha$ we use for learning $Q_\theta$ (Equation~\ref{unbiased_bellman}) and goal distribution densities $f_\theta$ (Equation~\ref{f_loss_gradient}) can be different. For discrete environments, we can learn the future distribution of the goal state using simple tabular methods, such as tabular successor representations. 

\vspace{-0.2cm}
\newcommand{\pluseq}{\mathrel{+}=}
\section{Algorithm and Implementation}
\vspace{-0.2cm}
USHER may be implemented atop DDPG \cite{DDPG}, SAC \cite{SAC}, TD3 \cite{TD3}, or any other continuous RL algorithm, as it only changes the loss function for training the goal-conditioned Q-value network. In our experiments, we use SAC as a base.
USHER calculates the loss as follows: It samples a batch of transitions $(s, a, s', \gp, T)$ from the replay buffer, along with two sets of goals: $\gr$, which is drawn from the future distribution of $s$, and $\gr'$, which is drawn uniformly from the goal space $G$. For each set of goals, we calculate two values of $W$. For weighting the $Q$-values, we use $\alpha_Q=0.01$, and for weighting the $f$-values, we use $\alpha_f=0.5$.
We omit the full training loop here, as it is identical to standard HER except for the loss computation.

We make a few minor adjustments to the derived formula in order to minimize the variance induced by importance sampling and ensure numerical stability.
In order to minimize the variance induced by importance sampling, we clip the importance sampling fraction to the range $[\frac{1}{1+c}, 1+c]$, where $c$ is a hyperparameter. This allows us to make a bias/variance trade-off between hindsight bias and the variance induced by importance sampling. 
We find that performance is best for $c \approx 0.3$, and that the bias induced by clipping is negligible for $c > 1.0$ for most environments. We apply this to the importance sampling weights for $Q$ only ($W_{\alpha_Q}$ and $W_{\alpha_Q}'$), not $f$. 
We approximate $W$ using $f_\theta$ for all experiments. In order to reduce the total number of neural network evaluations, we made $Q_\theta$ and $f_\theta$ two heads of a two-headed neural network. Although this choice conditions the value function on $T$, note that by the definition of the $Q$ function, $Q(s, a, \gr, \gp) = \mathbb{E}[\sum_{t=0}^\infty \gamma^t R^t \mid s, a, \gr, \gp] = \mathbb{E}_T[\mathbb{E}[\sum_{t=0}^\infty \gamma^t R^t \mid s, a, \gr, \gp, T]] =  \mathbb{E}_T[Q(s, a, \gr, \gp, T)]$. Thus, training the policy on the $T$-conditioned value function should result in the policy receiving the same gradient in expectation. We found that this trick did not noticeably impact the behavior of the $Q$ function or the policy.

\begin{algorithm}[tb]
\small
\thinmuskip=0mu
  \caption{USHER}
  \label{alg:example}
\begin{algorithmic}
  \STATE {\bfseries Input:} Replay Buffer $B$, Two-headed Critic Network with weights $\theta$, Actor Network with weights $w$, Weighting Factor $\alpha_Q$ for $Q$ , Weighting Factor $\alpha_f$for $f$, Goal Space $G$, Goal Function $\phi$;
  \STATE Sample batch of tuples $(s, a, s', \gp, T)$ from $B$; $critic\_loss \leftarrow 0$; $actor\_loss  \leftarrow 0$;
  \FOR{{\bf each} sampled tuple $(s, a, s', \gp, T) \in B$ }
    \STATE Sample $\gr$ as $\gp$, with probability $\frac{1}{k+1}$, and uniformly from the future trajectory that starts at $s$ with probability $\frac{k}{k+1}$ ($k$ is a pre-defined number); Sample alternative goal $\gr'\sim \textrm{Uniform}(G)$;
  \STATE $target\_q =R(\phi(s'), \gr) + Q_{\textrm{target}}(s', \pi(s', \gp), \gr, \gp, T - 1)$; // $Q_{\textrm{target}}$ is a copy of $Q_{\theta}$
  \STATE $target\_q' =R(\phi(s'), \gr') + Q_{\textrm{target}}(s', \pi(s', \gp), \gr', \gp, T - 1)$; 
  \STATE Define $W(s, a, s', \gp, \gr, T, \alpha) = \frac{f_\theta(\gr \mid s, a, \gp, T)}{\alpha f_\theta(\gr \mid s, a, \gp, T) + (1-\alpha) f_{\textrm{target}}(\gr \mid s', \pi(s', \gp), \gp, T - 1)}$;
    \STATE Set $W_{\alpha_f} = (1-\alpha_f) W(s, a, s', \gp, \gr, T, \alpha_f)$,
    $W_{\alpha_f}' = \alpha_f W(s, a, s', \gp, \gr', T, \alpha_f)$,
    $W_{\alpha_Q} = (1-\alpha_Q) W(s, a, s', \gp, \gr, T, \alpha_Q)$, and
    $W_{\alpha_Q}' = \alpha_Q W(s, a, s', \gp, \gr', T, \alpha_Q)$;

\STATE $critic\_loss \mathrel{+}= \Big(W_{\alpha_f} \big(f_\theta(\gr \mid s, a, \gp, T)^2  - 2f_\theta(\gr \mid s, a, \gp, T)f_{\textrm{target}}(\gr \mid s', \pi(s', \gp), \gp, T)\big) \Big)$ 
$ + \Big( W_{\alpha_f}' \big(f_\theta(\gr' \mid s, a, \gp, T)^2 - 2f_\theta(\gr' \mid s, a, \gp, T)f_{\textrm{target}}(\gr' \mid s', \pi(s', \gp), \gp, T - 1)\big) \Big)$
$ - \Big(\frac{2}{T} f_\theta(\phi(s') \mid s, a, \gp, T) \Big)$
$ + \Big( W_{\alpha_Q} (Q_\theta(s, a, \gr, \gp, T) - target\_q)^2 \Big)$
$ + \Big( W_{\alpha_Q}' (Q_\theta(s, a, \gr', \gp, T) - target\_q')^2 \Big)$;
  \STATE $actor\_loss \mathrel{+}= -Q_\theta(s, \pi(s, \gp), \gp, \gp, T)$;
  \ENDFOR
  \STATE Backprop $critic\_loss$ and update $\theta$; Backprop $actor\_loss$ and update $w$;
\end{algorithmic}
\end{algorithm}



\vspace{-0.3cm}
\section{Experiments}
\vspace{-0.3cm}
{\bf Tasks.} We evaluate the proposed algorithm in the following environments, illustrated in Figure~\ref{fig:tasks}. (1) {\it Discrete}. We first demonstrate our method in the discrete case in Figure~\ref{fig:test} because it is analytically tractable and allows us to verify that USHER learns the correct value function. The environment used has a short, risky path that has a high chance of disabling the robot, and a longer risk-free path. The longer path has a higher expected reward, but we will demonstrate that HER mistakenly prefers the riskier path. 
(2) \textit{4-Torus with Freeze}. This environment was introduce in~\citep{UnbiasedMethods} to demonstrate HER's bias. Robots take steps on a continuous N-dimensional torus to reach a location. They also have a "freeze" action that will teleport them to a random location, but freeze them in place for the rest of the episode. HER learns to always take the freeze action. 
(3) \textit{Car with Random Noise} and (4) \textit{Red Light}. These two environments use the "Simple Car" dynamics described in~\citep{lavalle:2006}. In \textit{Car with Random Noise}, the robot must navigate around walls while subject to  Gaussian action noise \cite{PlanningAlgorithms}. In \textit{Red Light}, the robot must learn to safely navigate a traffic light to reach its goal. 
(5) \textit{Fetch}. These three environments task a robot manipulating a robot arm to reach a point, push an object, and slide an object to a point outside of the robot's reach, respectively. 
We use these to demonstrate the USHER has comparable performance to HER on deterministic, high-dimensional environments. 
(6) \textit{Mobile Throwing Robot}. We design a simulated robot arm on a mobile base, and task it with throwing a ball to a randomly selected location. There is also a 50\% chance of wind that can blow the ball off course. 
(7) \textit{Navigation on a physical mechanum robot}. Lastly, we train a mechanum robot to navigate around obstacles to reach a goal and deploy it on a physical robot. The terrain contains a high friction zone that leads to the goal faster, but unreliably. Transfer was done by rolling out trajectories in simulation, and then deploying the same sequence of actions on the physical robot as an open-loop control. 
\begin{figure*}
	\begin{tabular}{@{\hskip0pt}c@{\hskip0pt}c@{\hskip0pt}c@{\hskip0pt}}
	  \vspace{-0.05cm}
 		\includegraphics[width=0.28\textwidth]{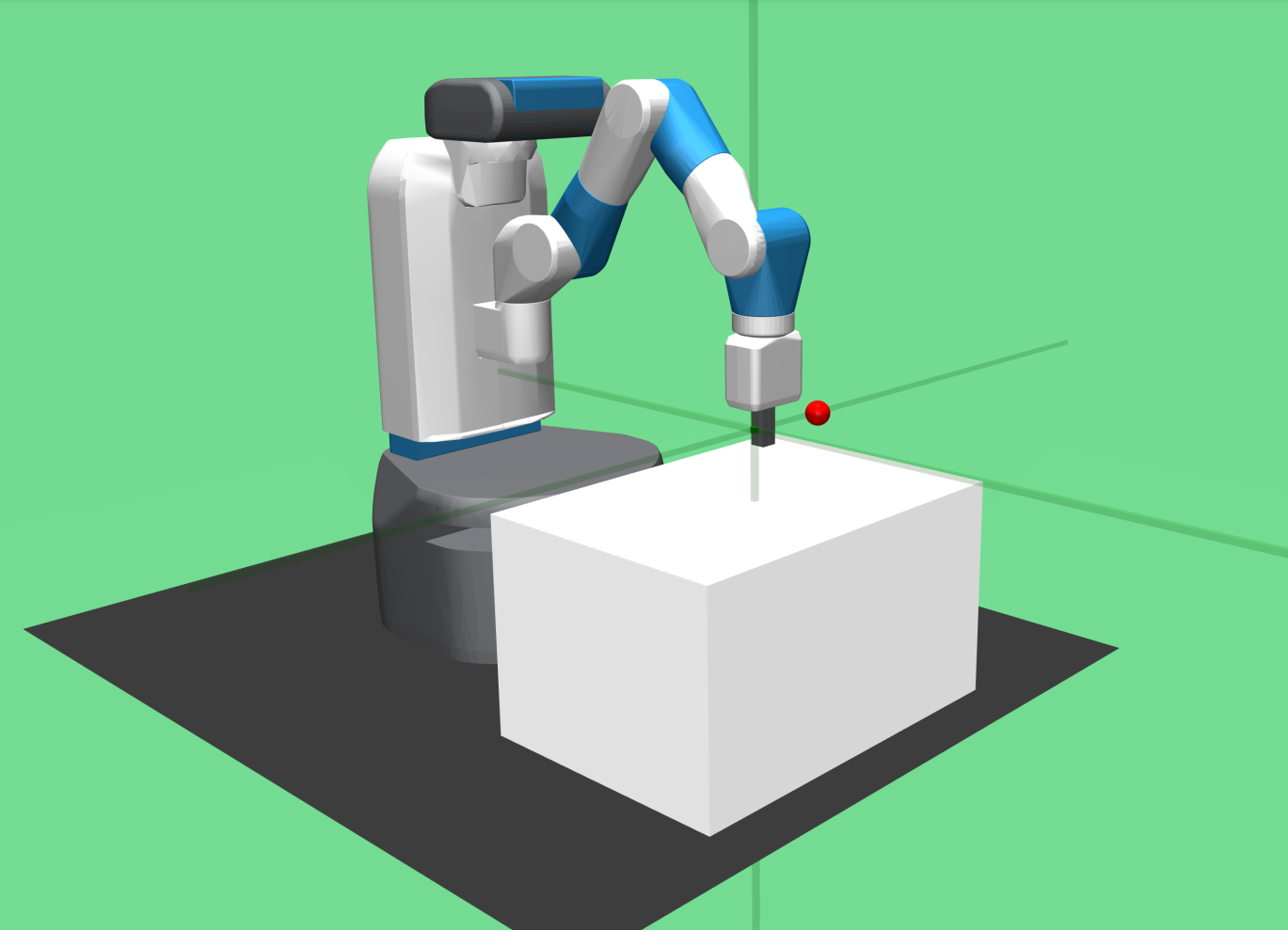} &
 		\hspace{-0.1cm}
 		\includegraphics[width=0.28\textwidth]{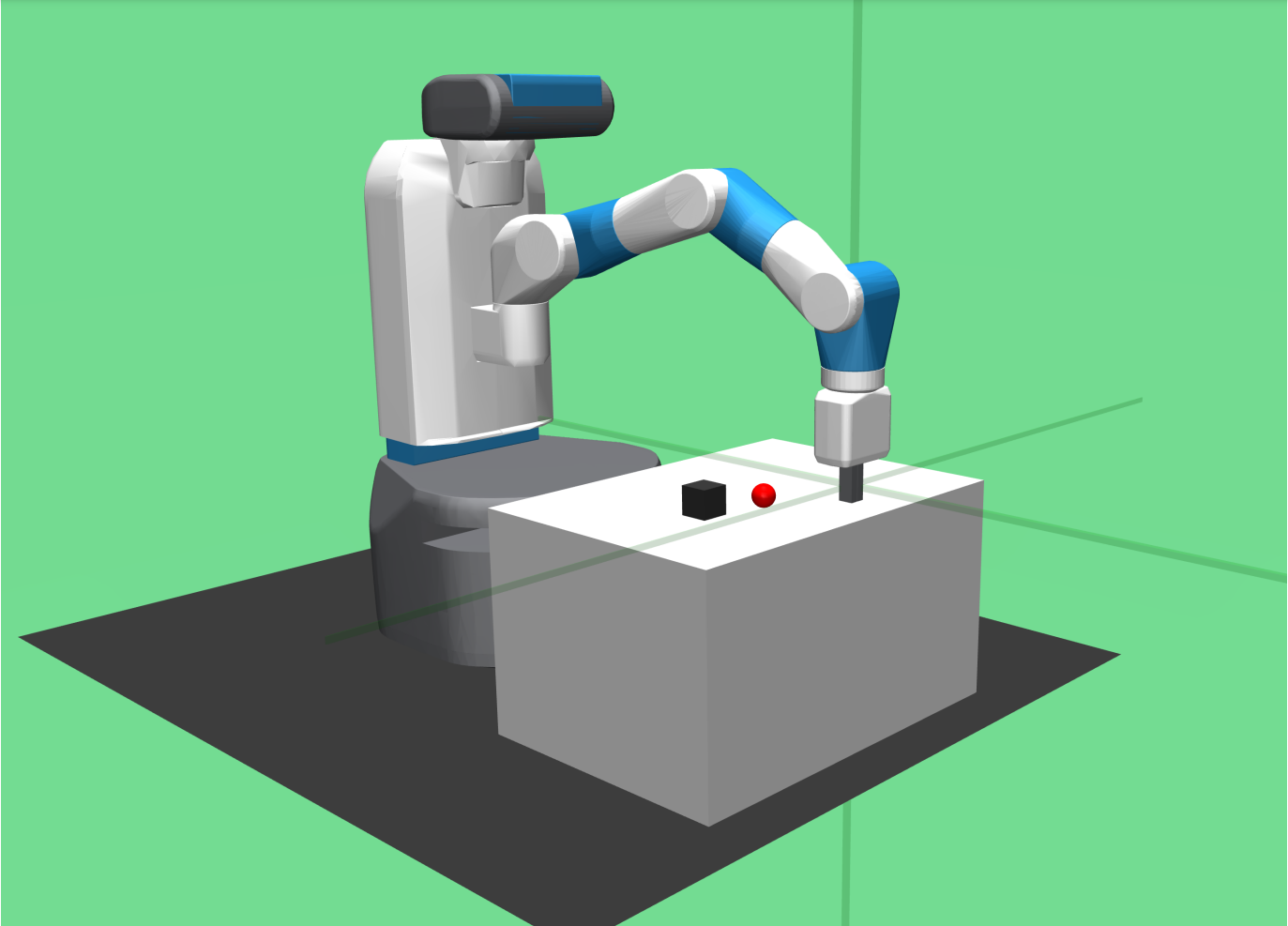}  &
 		\hspace{-0.1cm}
 		\includegraphics[width=0.28\textwidth]{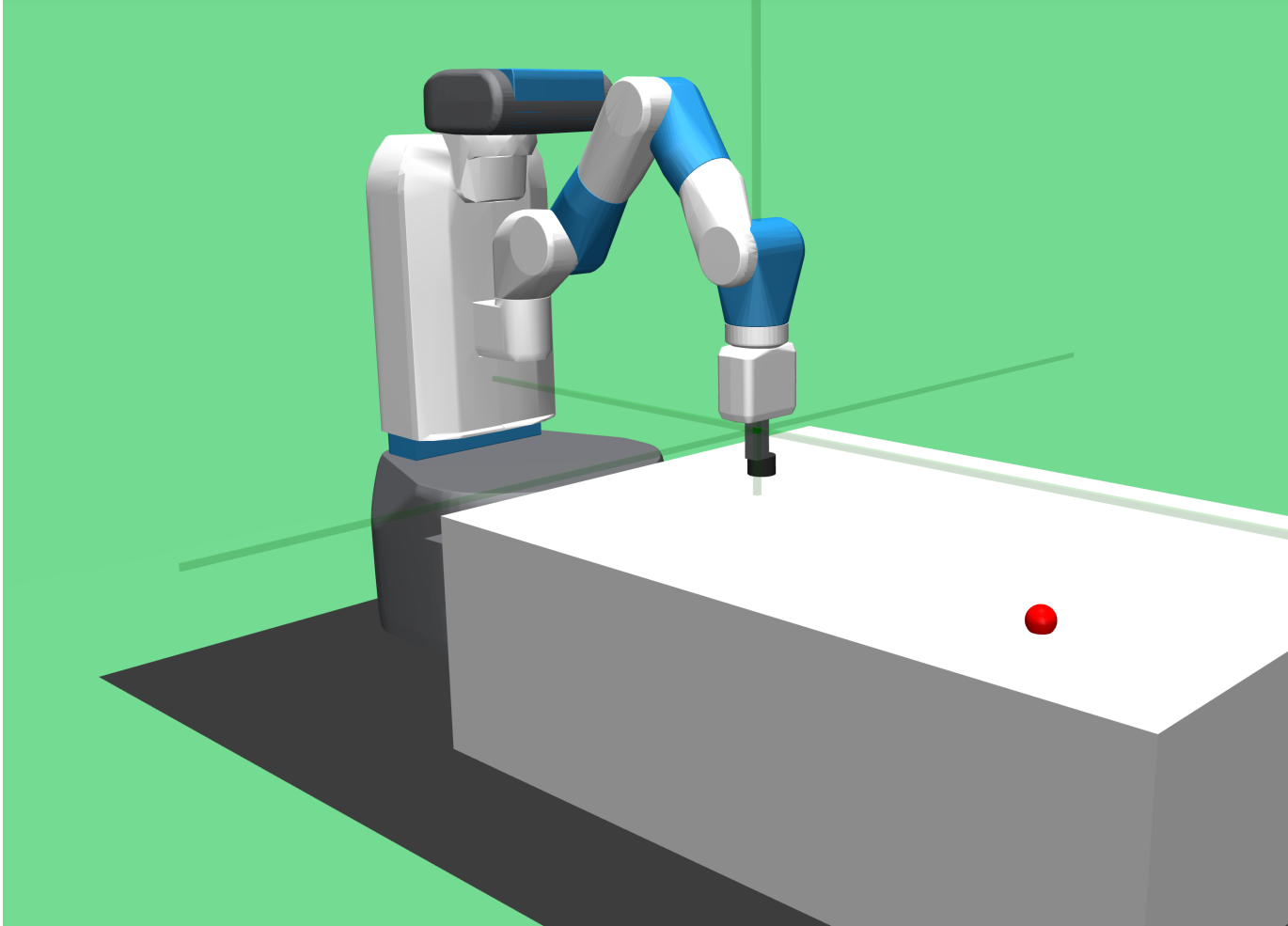} \\
 		\vspace{-0.05cm}
 		(a) & (b) & (c) \\
 		 		\includegraphics[width=0.28\textwidth]{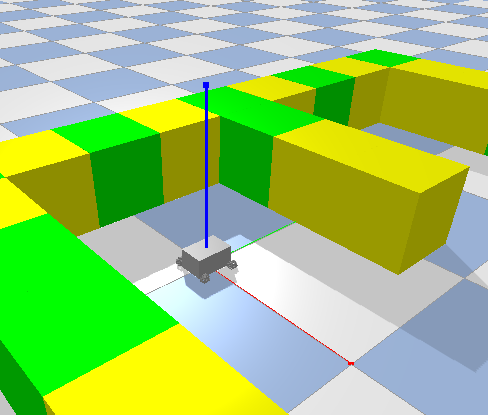} &
 		\hspace{-0.1cm}
 		\includegraphics[width=0.28\textwidth]{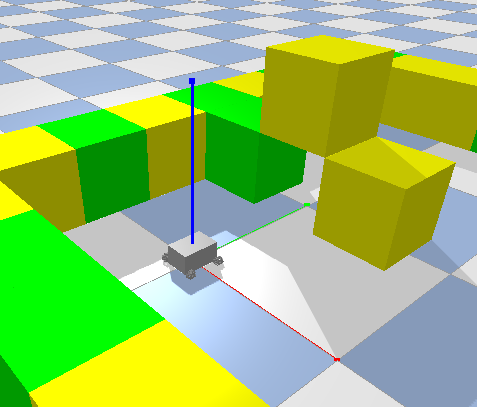}  &
 		\hspace{-0.1cm}
 		\includegraphics[width=0.25\textwidth]{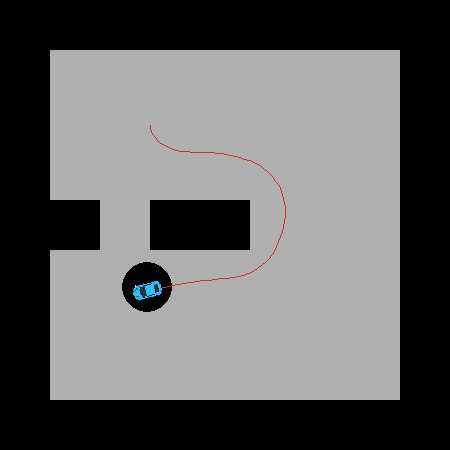} \\
 		\vspace{-0.05cm}
 		(d) & (e) & (f) \\
 		 \includegraphics[width=0.32\textwidth]{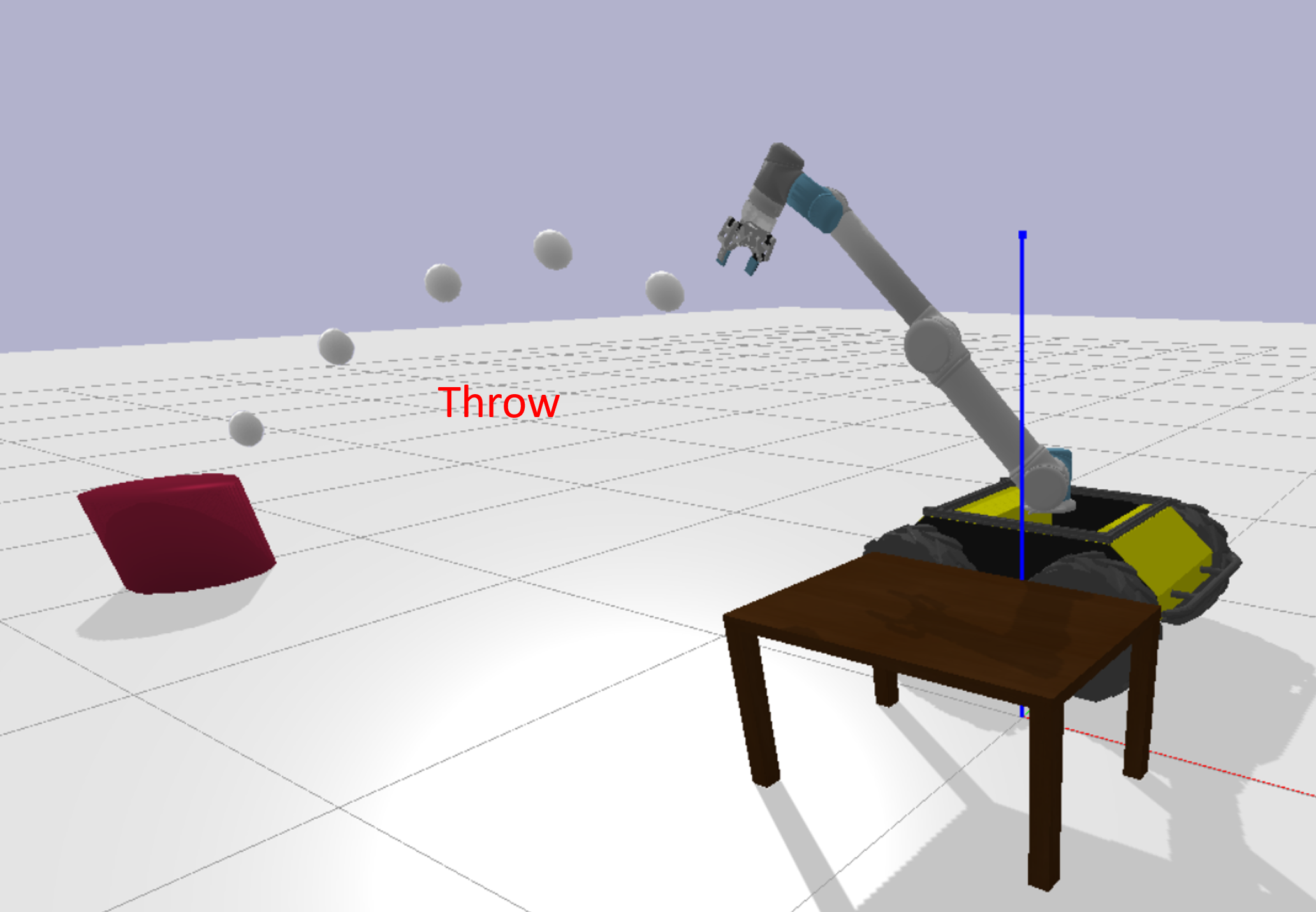} &
 		\hspace{-0.1cm}
 		\includegraphics[width=0.32\textwidth]{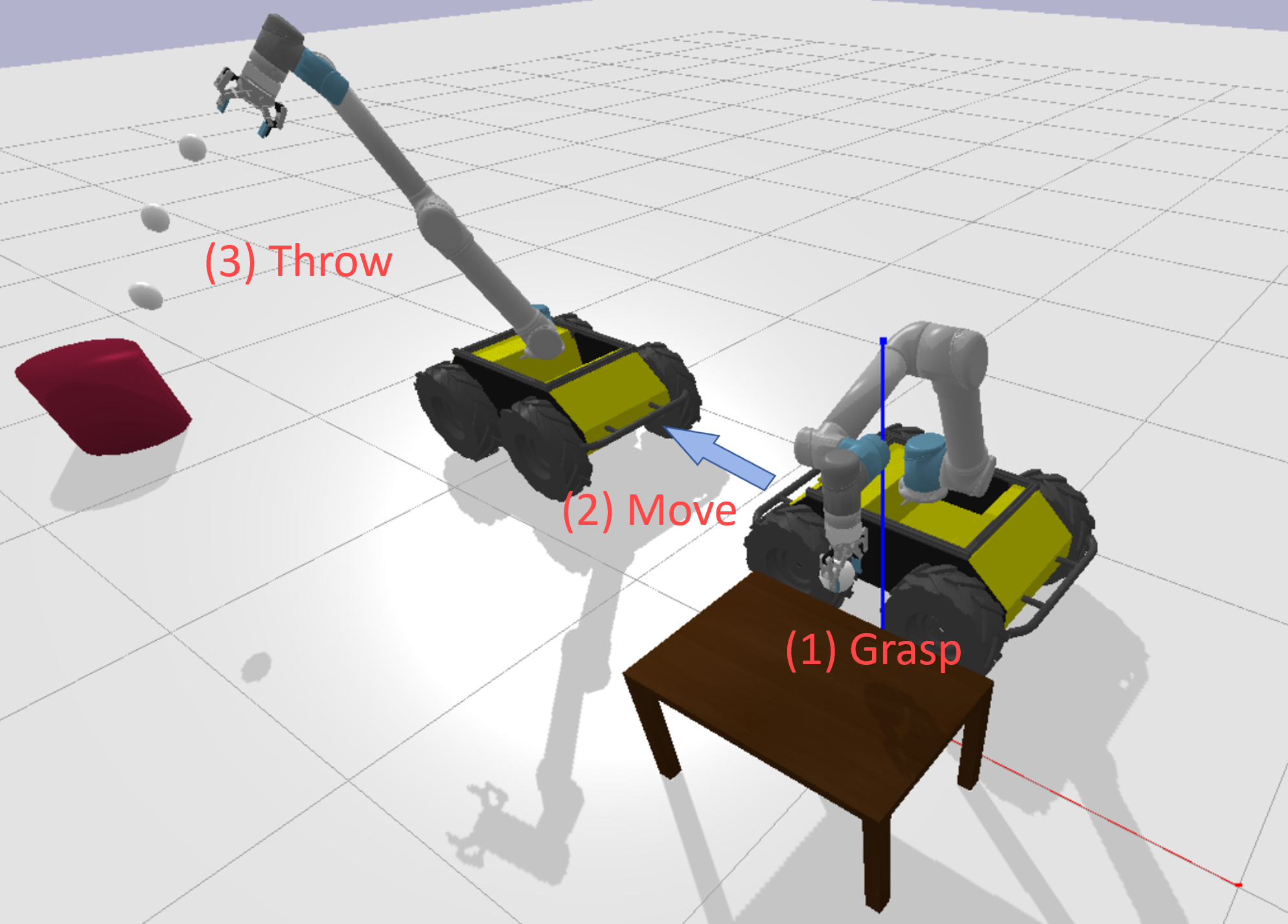}  &
 		\hspace{-0.1cm}
 		\includegraphics[width=0.34\textwidth]{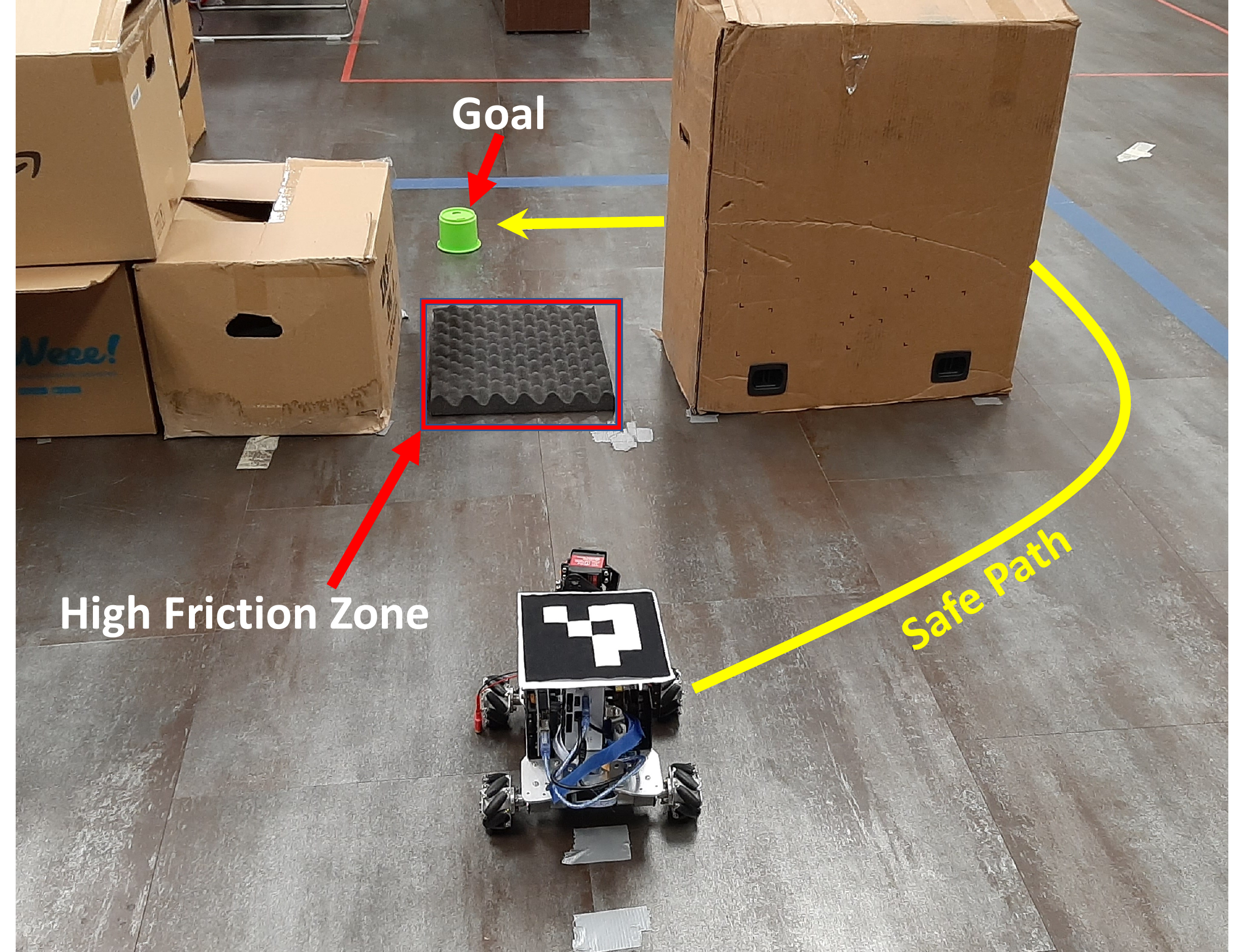} \\
 		\vspace{-0.05cm}
 		(g) & (h) & (i) \\
	\end{tabular}
    \caption{\small Some of the robotic tasks used in the experiments. (a) Fetch reach. (b) Fetch push. (c) Fetch slide. (d,e)  Simulated mechanum robot with random obstacle in two different positions. (f) Car with random noise.  (g, h) Mobile Throwing Robot. (i) Physical Mechanum Robot.}
    \vspace{-0.5cm}
    \label{fig:tasks}
    \vspace{-0.2cm}
\end{figure*}

{\bf Evaluation Metrics.}
On all tasks, we report the success rate, the average reward, and the average bias at the starting state (measured as the average cumulative discounted reward minus the value at the starting state). For all experiments (except Discrete and physical and simulated Mechanum environments), we report the sample mean and confidence interval, measured using five training runs for each method with randomly initialized seeds. 

{\bf Summary of Results.} Figure~\ref{fig:quan_all} summarizes the results of the experiments.  
(1) \textit{Discrete}. The value functions for USHER and Q-learning both quickly converge to the expected value, while HER overestimates the expected reward. USHER and Q-learning both learn to take the long, safe path, while HER takes the short, risky path. 
(2) \textit{4-Torus with Freeze}. HER learns to always take the freeze action and fails as a result, while USHER learns a successful policy. DDPG and $\delta$-DDPG are unbiased in this environment, but DDPG struggles due to the difficulty of exploring in high dimensions, and $\delta$-DDPG struggles with its variance.
(3) \textit{Car with Random Noise}. HER performs well for low noise values, but tends to overestimate values more as the noise level rises. USHER suffers significantly less from high noise levels than HER. 
(4) \textit{Red Light}. HER learns to run the red light, while USHER waits for the red light to end. USHER achieves higher success rates and rewards. 
(5) \textit{Fetch}. USHER is able to match HER's performance on all of the tested environments.  This suggests that the importance sampling method does not significantly affect USHER's variance or sample efficiency in deterministic environments, where HER is known to be unbiased. It also significantly outperforms two other unbiased methods, DDPG and $\delta$-DDPG on FetchReach. 
(6) \textit{Mobile Throwing Robot}. USHER matches HER's sample efficiency until the point where HER's bias causes its performance to suffer. USHER's performance, by contrast, continues to grow steadily to a 75\% success rate, significantly better than HER's 55\%. Interestingly, we find that USHER actually underestimates its reward here. This is likely because this environment is slightly non-Markovian, because the wind is sampled at the beginning of each trajectory, and then remains fixed. USHER's proof of unbiasedness, however, assumes that the environment is Markovian. It is interesting to note that USHER still performs well even when this property does not completely hold. 
(7) \textit{Navigation on a physical mechanum robot}. We find that USHER outperforms HER. Both robots take the short goal when it is open. When the path is blocked, HER repeatedly slams into the obstacle. By contrast, USHER runs into the block once, and then turns to go around it if it is blocked. This leads USHER to have a higher success rate. In simulation, HER's success rate is approximately 50\%, while USHER's is near 100\%. Due to the difficulty of transfer, USHER's performance drops on the physical robot, but it still outperforms HER. HER succeeded on 4/10 goals, while USHER succeeds on 6/10. Both methods succeed 100\% of the time on the unblocked path environment. 


\begin{figure*}[h]
	\begin{tabular}{@{\hskip0pt}c@{\hskip0pt}c@{\hskip0pt}c@{\hskip0pt}c@{\hskip0pt}}
	  \hspace{-0.3cm}\vspace{-0.05cm}
 		\includegraphics[width=0.28\textwidth]{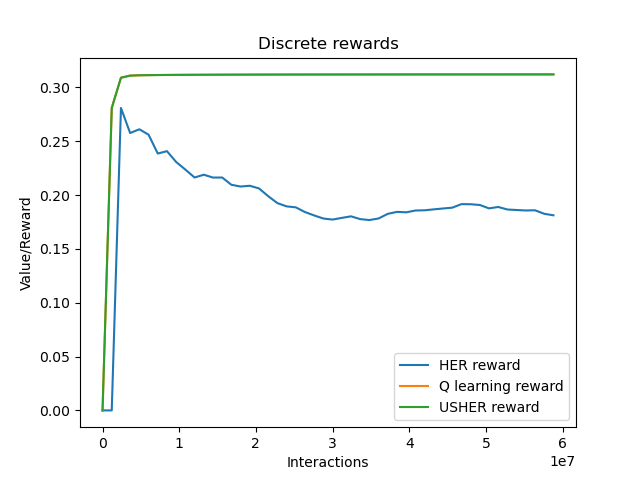} &
 		\hspace{-0.45cm}
 		\includegraphics[width=0.28\textwidth]{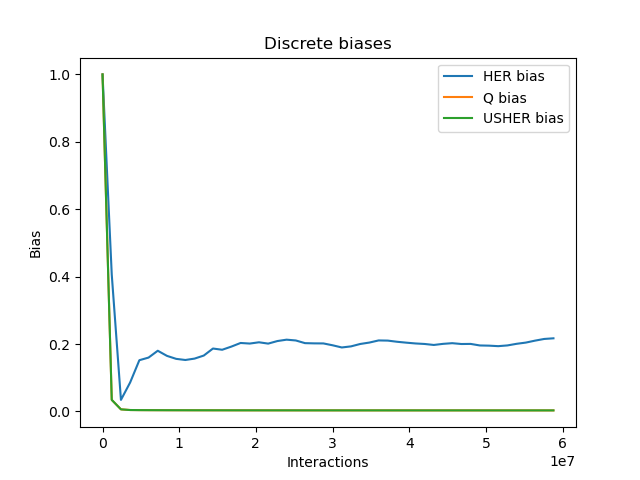} &
 		\hspace{-0.45cm}
 		\includegraphics[width=0.28\textwidth]{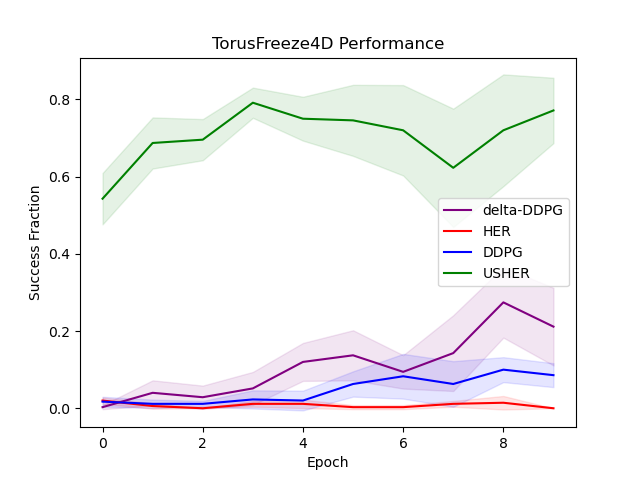} &
 		\hspace{-0.45cm}
 		\includegraphics[width=0.28\textwidth]{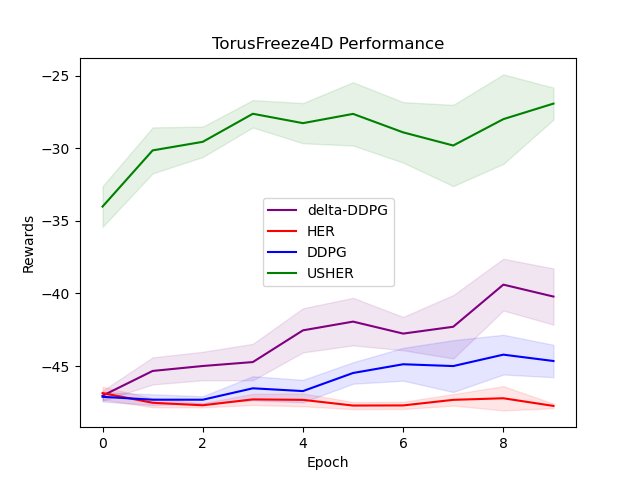} \\
 		\vspace{-0.05cm}
 		(a) & (b) & (c) & (d)\\
 		\hspace{-0.3cm}
		\includegraphics[width=0.28\textwidth]{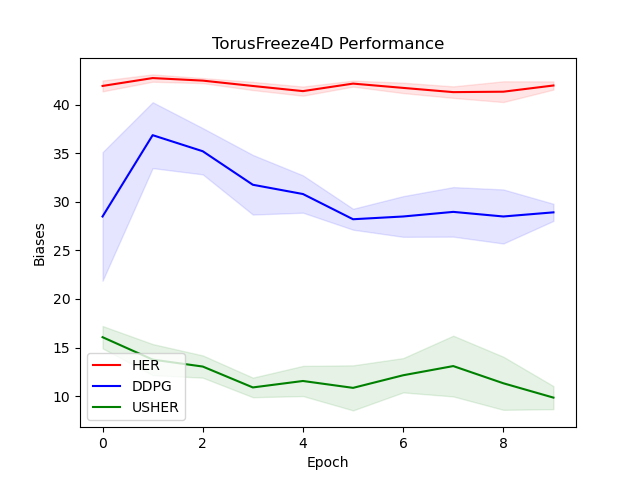} &
		\hspace{-0.45cm}
 		\includegraphics[width=0.28\textwidth]{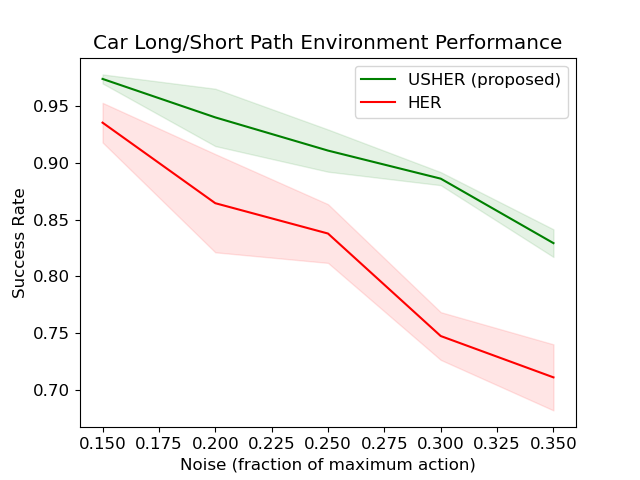} &
 		\hspace{-0.45cm}
 		\includegraphics[width=0.28\textwidth]{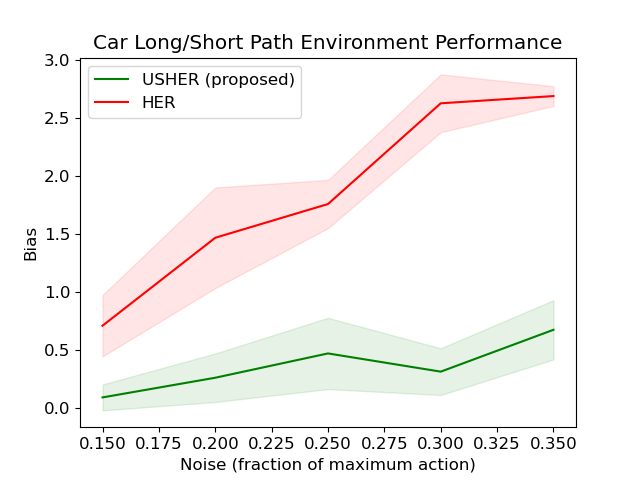} &
 		\hspace{-0.45cm}
 		\includegraphics[width=0.28\textwidth]{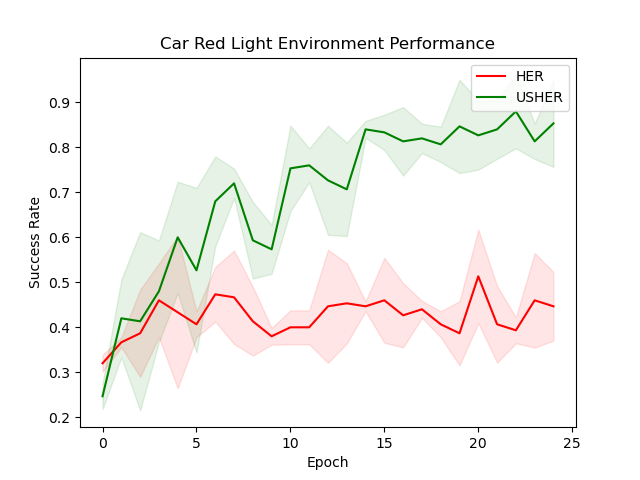}\\
 		\vspace{-0.03cm}
 		(e) & (f) & (g) & (h)\\
 		\hspace{-0.3cm}
 		\includegraphics[width=0.28\textwidth]{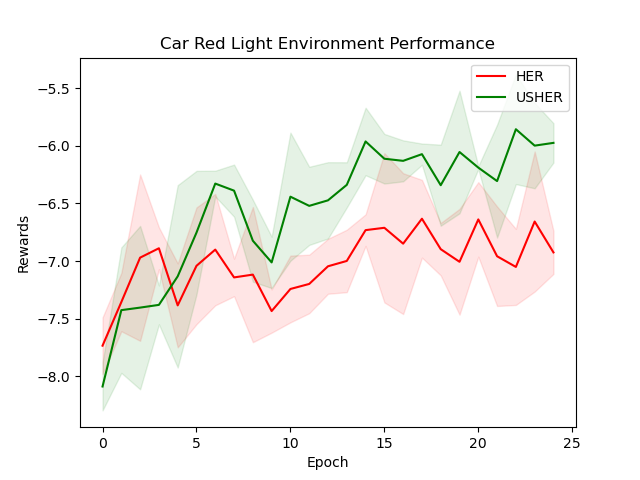} &
 		\hspace{-0.45cm}
 		\includegraphics[width=0.28\textwidth]{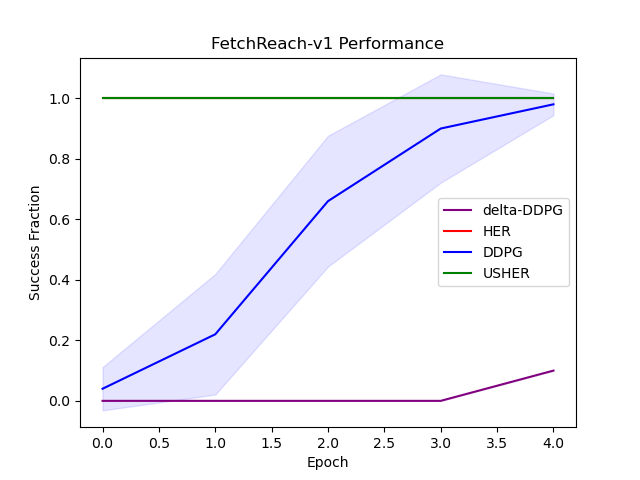} &
 		\hspace{-0.45cm}
		\includegraphics[width=0.28\textwidth]{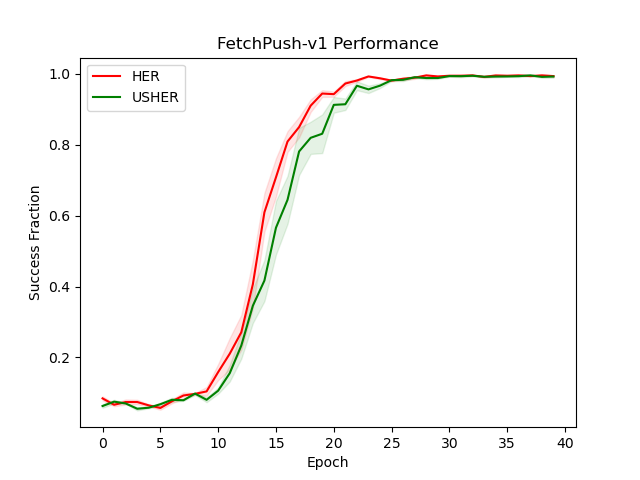} &
		\hspace{-0.45cm}
 		\includegraphics[width=0.28\textwidth]{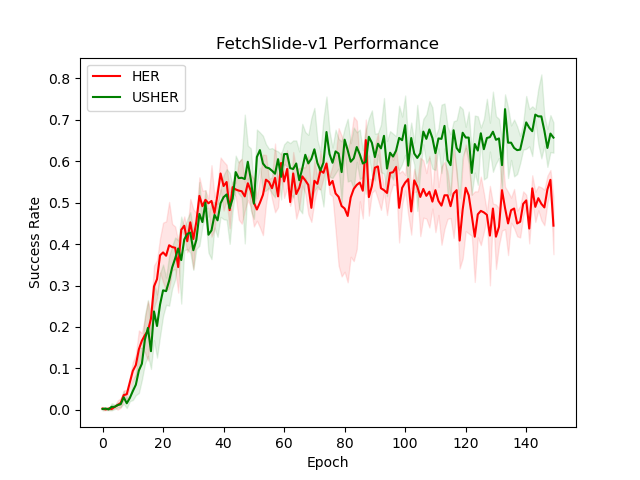} \\
 		\vspace{-0.05cm}
 		(i) & (j) & (k) & (l)\\
 		\hspace{-0.3cm}
		\includegraphics[width=0.28\textwidth]{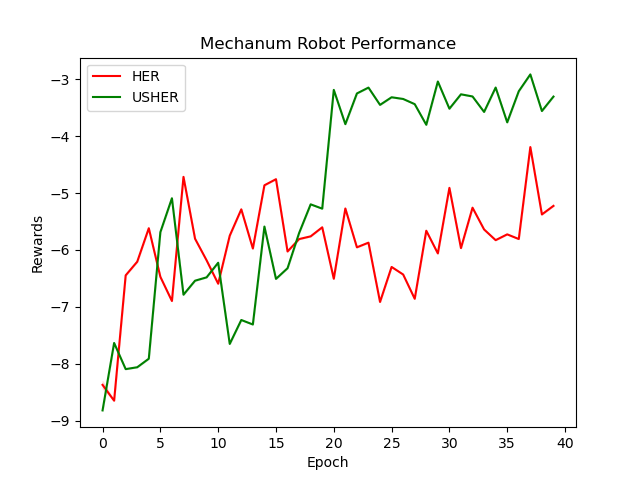} &
		\hspace{-0.45cm}
		\includegraphics[width=0.28\textwidth]{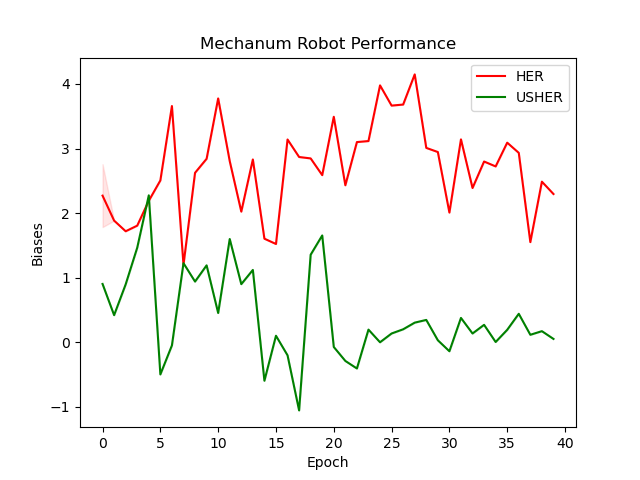} &
		\hspace{-0.45cm}
		\includegraphics[width=0.28\textwidth]{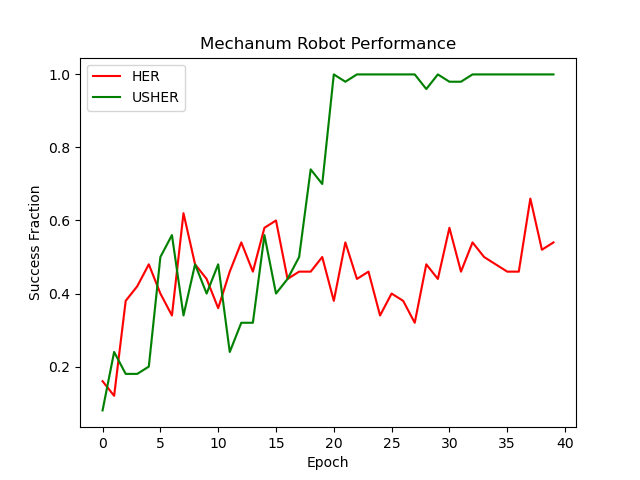} &
		\hspace{-0.45cm}
		\includegraphics[width=0.28\textwidth]{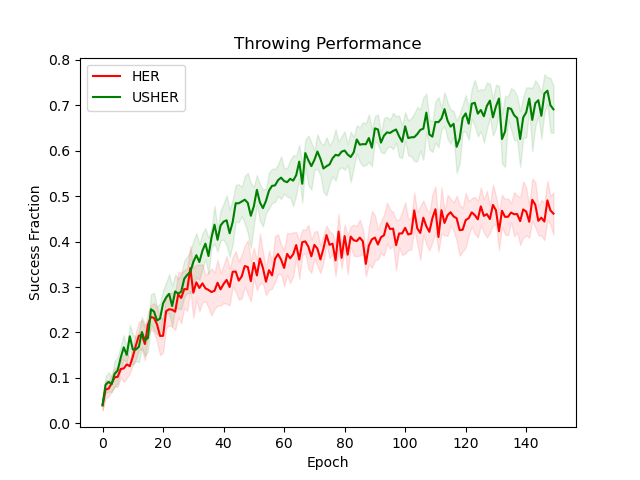} \\
		\vspace{-0.05cm}
 		(m) & (o) & (p) & (q)\\
 		 		\hspace{-0.3cm}
		 &
		\hspace{-0.45cm}
		\includegraphics[width=0.28\textwidth]{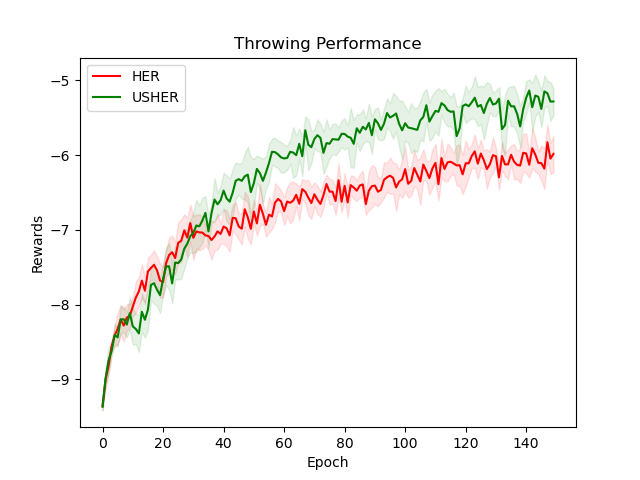} &
		\hspace{-0.45cm}
		\includegraphics[width=0.28\textwidth]{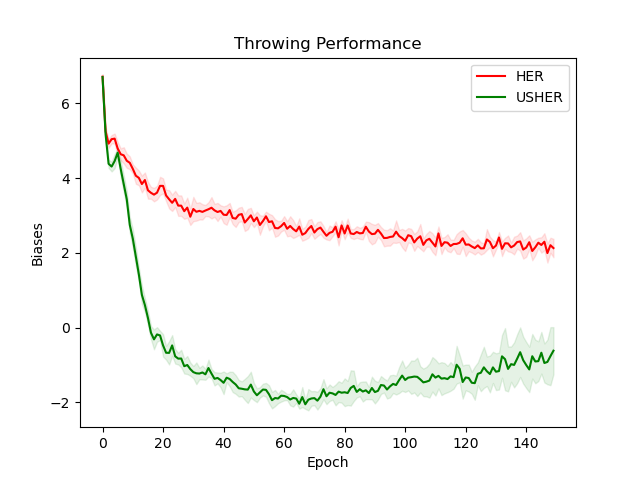}\\
		\vspace{-0.05cm}
 		 & (r) & (s) &\\
	\end{tabular}
\vspace{-0.2cm}
    \caption{Average reward, success rate, and Q-value bias of HER and USHER on the different tasks.}
    \vspace{-0.7cm}
    \label{fig:quan_all}
\end{figure*}

\vspace{-0.2cm}
\section{Limitations}
\vspace{-0.2cm}
One limitation of this work is that we rely on the Markov assumption to derive our importance sampling weights. This means that while we can correctly estimate the value function for stochastic transitions, we cannot guarantee that the learned value is correct in environments with hidden information. It is unclear whether this is actually an issue in practice, as USHER still outperforms HER on the non-Markovian environments we tested (such as the Throwing Bot). Additionally, USHER requires approximately $2.5$ times as many neural net evaluations as HER does per batch update. This was not an issue in our experiments, as the cost of simulation and policy evaluations usually dominated the training time. 
\vspace{-0.3cm}
\section{Conclusion}
\vspace{-0.3cm}
We derive an unbiased importance sampling method for HER, and show that it is able to effectively counteract HER's hindsight bias. We find that addressing this bias leads to higher success rates and rewards in a range of stochastic environments.  Furthermore, we introduce a mathematical framework to justify our method which can be used to examine the situations where HER is likely to experience significant bias. In future work, we hope to examine the finite-sample case, in order to better understand whether HER introduces a bias there, and if so, how it could be corrected. 




\clearpage


\bibliography{references}  

\pagebreak

\appendix

\section{Appendix}


\subsection{Experimental Details}
All experiments were performed on an Alienware-Aurora-R9 with 8 Intel i7-9700 cores. We did not use a GPU. Code for the experiments is provided in the supplementary materials

\subsubsection{Discrete Case} 
We perform 1000 episodes, using trajectories of 30 steps, a $k$ of 8, and a $\gamma$ of 0.825. Our learning rate begins at 0.01 and decays as $O(n^{-\frac{3}{4}})$.

\subsubsection{For all continuous-space experiments}
NNs: We use four-layer neural nets with RELU activations and 256 hidden units for both the actor and critic networks \\
Optimizer: Adam with default params\\
Learning rate: .001 for both actor and critic \\
Polyak averaging coefficient: .95 \\
K (number of hindsight goals per non-hindsight goal): 8 \\
Batch-size: 256\\
$\alpha_Q$: 0.1 \\
$\alpha_f$: 0.5 \\

\textit{N-torus with Freeze}\\
$\gamma: .98$ \\ 
Trajectory length: 50 \\ 
Batches per epoch: 4\\
Episodes per epoch: 500\\
Entropy Regularization: 0.001 \\
\textit{Car with Random Noise}\\
$\gamma: .95$ \\ 
Trajectory length: 20 \\ 
Batches per epoch: 5\\
Episodes per epoch: 500\\
Entropy Regularization: 0.01 \\
\textit{RedLight}\\
$\gamma: .9$ \\ 
Trajectory length: 50 \\ 
Batches per epoch: 4\\
Episodes per epoch: 500\\
Entropy Regularization: 0.01 \\
The light pattern was green: 1 second, yellow: 1 second, red: 4 seconds, with a randomized starting color. 
\textit{FetchReach}\\
$\gamma: .98$ \\ 
Trajectory length: 50 \\ 
Batches per epoch: 40\\
Episodes per epoch: 50\\
Entropy Regularization: 0.001 \\
\textit{FetchPush}\\
$\gamma: .98$ \\ 
Trajectory length: 50 \\ 
Batches per epoch: 40\\
Episodes per epoch: 50\\
Entropy Regularization: 0.01 \\
\textit{FetchSlide}\\
$\gamma: .98$ \\ 
Trajectory length: 50 \\ 
Batches per epoch: 40\\
Episodes per epoch: 50\\
Entropy Regularization: 0.001 \\
\textit{Mobile Throwing Robot}\\
$\gamma: .9$ \\ 
Trajectory length: 20 \\ 
Batches per epoch: 100\\
Episodes per epoch: 50\\
Entropy Regularization: 0.001 \\
\textit{Mechanum robot -- simulator}\\
$\gamma: .925$ \\ 
Trajectory length: 50 \\ 
Batches per epoch: 100\\
Episodes per epoch: 50\\
Entropy Regularization: 0.01 \\
\textit{Mechanum robot -- analytic model}\\
$\gamma: .975$ \\ 
Trajectory length: 50 \\ 
Batches per epoch: 10\\
Episodes per epoch: 50\\
Entropy Regularization: 0.01 \\
Importance weight clipping value: 100\\

Our USHER and HER implementations are based on Tianhong Dai's implementation \cite{her_implementation}. 

\subsection{Implementation}
An implementation of USHER and our experiments can be found at: 
\url{https://anonymous.4open.science/r/USHER_CoRL-0E16/README.md}


\subsection{Goal Selection Probability}
\textbf{Proposition 1:} \\
Suppose $\gp$ is fixed at the start of the trajectory, and $\gr$ is sampled using HER. Then for any $s', s, a, \gr, \gp, T$,
\begin{align*}
    f(s' \mid s, a, \gr, \gp, T) = \frac{f(\gr \mid s', \pi(s', \gp), \gp, T-1) }{f(\gr \mid s, a, \gp, T)} f(s' \mid s, a)
\end{align*}

\begin{proof}
Suppose $\gp$ is sampled before the trajectory begins, and is not changed at training time.
Let $s, a$, and $s'$ be random variables representing a state, action, and subsequent state. 
Let $Q_{HER}^{\pi}(s, a, \gr, \gp)$ be the solution to the Bellman equation obtained using HER's sampling bias, with state $s$, hindsight goal $\gr$, policy goal $\gp$, and deterministic policy $\pi(s', \gp)$. 
Let $t_s$ be the number of steps remaining in the trajectory when state $s$ is sampled, $t_s'$ be the number of steps remaining in the trajectory when state $s'$ is sampled, and $T$ be an integer

\begin{align*}
    f(s' \mid s, a, \gr, \gp, t_s = T) &=  \frac{f(s', s, a, \gr, \gp, t_s = T)}{f(s, a, \gr, \gp, t_s = T)} \\ 
    &=  \frac{f(\gr \mid s', s, a, \gp, t_s = T) f( s', s, a, \gp, t_s = T)}{f(s, a, \gr, \gp, t_s = T)} \\
    &=  \frac{f(\gr \mid s', s, a, \gp, t_s = T) f(s' \mid s, a, \gp, t_s = T)} 
    {f(\gr \mid s, a, \gp, t_s = T)} \frac{f(s, a, \gp, t_s = T)}{f(s, a, \gp, t_s = T)} \\
    &= f(s' \mid s, a, \gp, t_s = T) \frac{f(\gr \mid s', s, a, \gp, t_s = T) }
    {f(\gr \mid s, a, \gp, t_s = T)} \\
\end{align*}

Observe that with HER, $\gr$ is either selected from the trajectory beginning with $s'$, sampled independently of $s'$ or left the same as the original goal given to the policy. In all three cases, $f(\gr \mid s', s, a, \gp, t_s = T)  = f(\gr \mid s', \gp, t_s = T)$.
In the first case where $\gr$ comes from the future trajectory, the Markov property implies that given the most recent observed state $s'$, $\gr$ is independent of all earlier states and actions, including $s$ and $a$, so $f(\gr \mid s', s, a, \gp, t_s = T)  = f(\gr \mid s', \gp, t_s = T)$. 
In the second case where $\gr$ is sampled independently of the trajectory, so $f(\gr \mid s', s, a, \gp, t_s = T)  = f(\gr) = f(\gr \mid s', \gp, t_s = T)$ . 
In the third case, $\gr = \gp$, so $\gr$ has no dependence on $s, a$, or $s'$. 
In any case, $f(\gr \mid s', s, a, \gp, t_s = T)  = f(\gr \mid s', \gp, t_s = T)$. 

For the same reason that $\gr$ depends only upon $s'$ and not on $s$ when $s'$ is known, $\gr$ depends only on $t_{s'}$ and not $t_s$ when $t_{s'}$ is known. Thus we find that $f(\gr \mid s', \gp, t_s = T) = f(\gr \mid s', \gp, t_{s'} = T-1)$.

\begin{align*}
    f(s' \mid s, a, \gp, \gr, t_s = T) &=  f(s' \mid s, a, \gp, t_s = T) \frac{f(\gr \mid s', \gp, t_{s'} = T-1) }{f(\gr \mid s, a, \gp, t_s = T)} \\
    &=  f(s' \mid s, a, \gp, t_s = T) \frac{E_{a'}[ f(\gr \mid s', a', \gp, t_{s'} = T-1) \mid s' \gp] }{f(\gr \mid s, a, \gp, t_s = T)} \\
    &=  f(s' \mid s, a, \gp, t_s = T) \frac{f(\gr \mid s', \pi(s', \gp), \gp, t_{s'} = T-1) }{f(\gr \mid s, a, \gp, t_s = T)} \\
\end{align*}

Observe that from the Markov assumption of the environment, the transition probability depends only on $s, a$, and does not depend on $\gp$ nor $t_s$. $\gp$ is sampled before the trajectory begins, independently of all other random variables. From this we can see that $f(s' \mid s, a)$ is independent of $\gp$ and $t_{s}$. 

We can then conclude that for all $\gr, \gp$

\begin{align*}
    f(s' \mid s, a, \gp, \gr, t_{s} = T) = f(s' \mid s, a) \frac{f(\gr \mid s', \pi(s', \gp), \gp, t_{s'} = T-1) }{f(\gr \mid s, a, \gp, t_{s} = T)}
\end{align*}

For conciseness, we will abbreviate this to
\begin{align*}
    f(s' \mid s, a, \gp, \gr, T) = f(s' \mid s, a) \frac{f(\gr \mid s', \pi(s', \gp), \gp, T-1) }{f(\gr \mid s, a, \gp, T)}
\end{align*}
\end{proof}

\subsection{2-goal HER is asymptotically unbiased}

\begin{cor}
Suppose $Q_{HER}^{\pi}(s, a, \gp, \gp)$ satisfies the Bellman equation and the distribution of future achieved goals is absolutely continuous with respect to the goal space. Then $Q_{HER}^{\pi}(s, a, \gp, \gp) = Q^*(s, a, \gp)$ for all $s, a, \gp$, where $Q^*(s, a, \gp)$ is the optimal goal-conditioned $Q$ function. 
\end{cor}

\begin{proof}
Suppose that the distribution of future achieved goals is absolutely continuous with respect to the goal space. Furthermore, suppose the goal space is a continuous space of at least one dimension.  Then the probability of arriving exactly at any given goal $\gr$ given the policy goal $\gp$ is infinitesimal. This means that the only time there is a non-zero probability of having $\gr=\gp$ is when $\gr$ is not drawn from the distribution of future achieved goals and HER instead uses the same goal as during the data-gathering phase. 

 Let $P(\gr=\gp \mid s, a, \gp)$ be the probability that $\gp$ is selected as the reward goal. 
Then 
\begin{align*}
    P(\gr=\gp \mid s, a, \gp, T) &= P(\gr=\gp \mid s, a, \gp, T, H) P(H) \\&+ P(\gr=\gp \mid s, a, \gp, T, \neg H) P(\neg H) \\
     &= P(\gr=\gp \mid s, a, \gp, T, H) P(H) + 1 P(\neg H)
\end{align*}
Since $P(\gr=\gp \mid s, a, \gp, H)$ is infinitesimal and $P(\neg H)$ is not, this reduces to

\begin{align*}
    P(\gr=\gp \mid s, a, \gp, T) &= P(\neg H) = \frac{1}{k+1}
\end{align*}
Thus, 

\begin{align*} 
    Q_{HER}^{\pi}(s, a, \gp, \gp)&= E_{s'} [\frac{P(\gr=\gp \mid s',  \pi(s', \gp), \gp, T) }{P(\gr=\gp \mid s, a, \gp, T)} \\
    &(R(s', \gp) + \gamma Q_{HER}^{\pi}(s', \pi(s', \gp), \gp, \gp)) \mid s, a, \gr, \gp, T] \\
    &= E_{s'} [\frac{1/(k+1)}{1/(k+1)} \\
    &(R(s', \gp) + \gamma Q_{HER}^{\pi}(s', \pi(s', \gp), \gp, \gp))\mid s, a, \gr, \gp, T ] \\
    &= E_{s'} [(R(s', \gp) + \gamma Q_{HER}^{\pi}(s', \pi(s', \gp), \gp, \gp))\mid s, a, \gr, \gp, T ] 
\end{align*} 

Now, observe that $Q_{HER}^{\pi}(s, a, \gp, \gp)$ satisfies the one-goal Bellman equation. Since the Bellman equation has a unique solution, and $Q^*(s, a, g)$ is a solution, $Q_{HER}^{\pi}(s, a, \gp, \gp) = Q^*(s, a, \gp)$.
\end{proof}

\subsection{Importance Sampling for Mixed Sampling Method}

\textbf{Proposition 2:}\\
Let $W(s', s, a, \gr, \gp, T) = \frac{f(\gr \mid s, a, \gp, T)}{\alpha f(\gr \mid s, a, \gp, T) + (1-\alpha) f(\gr \mid s', \pi(s', \gp), \gp, T)}$. 
Let $\alpha$ be a real value in the range $(0,1]$. Then for any $s', s, a, \gr, \gp$, 
\begin{align*}
    f(s' \mid s, a) = W(s', s, a, \gr, \gp, T)( \alpha f(s' \mid s, a)+ (1-\alpha) f(s' \mid s, a, \gp, \gr, T)) 
\end{align*}

Furthermore, for any function $F$ of the state $s'$,
\begin{align*}
       \mathbb{E}_{s'}[F(s') \mid s, a] = \alpha &\mathbb{E}_{s'}[W(s', s, a, \gr, \gp, T) F(s') \mid s, a] \\ + (1-\alpha)&\mathbb{E}_{s'}[W(s', s, a, \gr, \gp, T) F(s') \mid s, a, \gp, \gr, T]
\end{align*}


\begin{proof}
Let $\gr, \gp, T$ be a reward goal, a policy goal, and the remaining steps left in the current trajectory, respectively
Using Proposition 1, we can show that
\begin{align*}
   f(s' \mid s, a) =&  f(s' \mid s, a) \frac{\alpha f(s' \mid s, a)+ (1-\alpha) f(s' \mid s, a, \gp, \gr, T)}{\alpha f(s' \mid s, a)+ (1-\alpha) f(s' \mid s, a, \gp, \gr, T)} \\ 
   =& ( \alpha f(s' \mid s, a)+ (1-\alpha) f(s' \mid s, a, \gp, \gr, T)) \\&\frac{f(s' \mid s, a)}{\alpha f(s' \mid s, a)+ (1-\alpha) f(s' \mid s, a, \gp, \gr, T)} \\ 
   =& ( \alpha f(s' \mid s, a)+ (1-\alpha) f(s' \mid s, a, \gp, \gr, T)) \\&\frac{f(s' \mid s, a)}{\alpha f(s' \mid s, a)+ (1-\alpha) \frac{f(\gr \mid s', \pi(s', \gp), \gp, T) }{f(\gr \mid s, a, \gp, T)}f(s' \mid s, a)} \\ 
   =& ( \alpha f(s' \mid s, a)+ (1-\alpha) f(s' \mid s, a, \gp, \gr, T)) \\&\frac{1}{\alpha+ (1-\alpha) \frac{f(\gr \mid s', \pi(s', \gp), \gp, T) }{f(\gr \mid s, a, \gp, T)}} \\ 
   =& ( \alpha f(s' \mid s, a)+ (1-\alpha) f(s' \mid s, a, \gp, \gr, T)) \\&\frac{f(\gr \mid s, a, \gp, T)}{\alpha f(\gr \mid s, a, \gp, T) + (1-\alpha) f(\gr \mid s', \pi(s', \gp), \gp, T)} \\
   =& ( \alpha f(s' \mid s, a)+ (1-\alpha) f(s' \mid s, a, \gp, \gr, T)) \\&W(s', s, a, \gr, \gp, T) \\
\end{align*}

It then follows that for any $\gr, \gp, T$, the expectation value $\mathbb{E}_{s'}[F(s') \mid s, a]$ may be written as follows:
\begin{align*}
    \mathbb{E}_{s'}[F(s') \mid s, a] =& \int_S f(s' \mid s, a) F(s') ds' \\
    =& \int_S ( \alpha f(s' \mid s, a)+ (1-\alpha) f(s' \mid s, a, \gp, \gr, T)) \\&W(s', s, a, \gr, \gp, T) F(s') ds' \\
    =& \int_S ( \alpha f(s' \mid s, a)+ (1-\alpha) f(s' \mid s, a, \gp, \gr, T)) \\&W(s', s, a, \gr, \gp, T) F(s') ds' \\
    =& \alpha \int_S f(s' \mid s, a) W(s', s, a, \gr, \gp, T) F(s') ds' \\&+ (1-\alpha) \int_S f(s' \mid s, a, \gp, \gr, T) W(s', s, a, \gr, \gp, T) F(s') ds' \\
    =& \alpha \mathbb{E}_{s'}[W(s', s, a, \gr, \gp, T) F(s') \mid s, a] \\&+ (1-\alpha) \mathbb{E}_{s'}[ W(s', s, a, \gr, \gp, T) F(s') \mid s, a, \gr, \gp, T] \\
    =&  \alpha \mathbb{E}_{s'}[W(s', s, a, \gr, \gp, T) F(s') \mid s, a] \\&+ (1-\alpha) \mathbb{E}_{s'}[ W(s', s, a, \gr, \gp, T) F(s') \mid s, a, \gr, \gp, T] \\
\end{align*}

\end{proof}

\end{document}